\documentclass[a4paper]{article}
\usepackage[a4paper]{geometry}

\usepackage{times}
\usepackage{graphicx,ifpdf,verbatim,multirow,pifont}
\usepackage[justification=centering]{caption}
\usepackage{subcaption}

\graphicspath{{figures/}}
\usepackage{ifpdf}
\ifpdf  \fi

\usepackage{natbib}
\renewcommand{\cite}{\citep}

\usepackage{algorithm,algorithmic}
\usepackage{hyperref}

\usepackage{amsmath,amsthm,amsfonts,amssymb,bm,mathtools}

\newcommand{\KL}{\mathrm{KL}}
\newcommand{\dx}{\mathrm{d}x}
\newcommand{\dy}{\mathrm{d}\bm{y}}
\newcommand{\dz}{\mathrm{d}z}

\newcommand{\relu}{\mathtt{relu}}

\newcommand{\diag}{\mathtt{diag}}

\newcommand{\fim}{\mathcal{G}}
\newcommand{\calI}{\mathcal{I}}

\def\calL{\mathcal{L}}

\DeclareMathOperator*{\argmin}{arg\,min}
\DeclareMathOperator*{\argmax}{arg\,max}

\usepackage{cleveref}
\crefname{subsection}{Subsec.}{Subsecs.}
\crefname{section}{Sec.}{Secs.}

\newtheorem{lemma}{Lemma}
\newtheorem{theorem}{Theorem}
\newtheorem{corollary}[theorem]{Corollary}
\newtheorem{remark}[theorem]{Remark}
\newcommand{\defeq}{\stackrel{def}{=}}
\allowdisplaybreaks

\newcommand*{\affaddr}[1]{#1} 

\newcommand*{\email}[1]{\texttt{#1}}

\begin{document} 




\title{Coarse Grained Exponential Variational Autoencoders}
\author{Ke Sun\hspace{1.5em}and\hspace{1.5em}Xiangliang Zhang\\
\affaddr{Computer, Electrical and Mathematical Sciences and Engineering Division\\
King Abdullah University of Science and Technology (KAUST)}\\
\email{sunk@ieee.org~~~xiangliang.zhang@kaust.edu.sa}}

\date{\today}
\maketitle

\begin{abstract}
Variational autoencoders (VAE) often use Gaussian or category distribution to model the inference process.
This puts a limit on variational learning because this simplified assumption does not match 
the true posterior distribution, which is usually much more sophisticated.
To break this limitation and apply arbitrary parametric distribution during inference,
this paper derives a \emph{semi-continuous} latent representation,
which approximates a continuous density up to a prescribed precision,
and is much easier to analyze than its continuous counterpart because it is fundamentally discrete.
We showcase the proposition by applying polynomial exponential family distributions as the posterior,
which are universal probability density function generators.
Our experimental results show consistent improvements over commonly used VAE models.
\end{abstract} 

\section{Introduction}\label{sec:intro}

Variational autoencoders~\cite{vae} and its variants~\cite{rezende14,sohn15,salimans15,burda16,serban17}
combine the two powers of variational Bayesian learning~\cite{jordan99}
with strong generalization and a standard learning objective, and deep learning with flexible and scalable representations.
They are attracting decent attentions,
producing state-of-the-art performance in semi-supervised learning~\cite{kingma14} and image generation~\cite{gregor15},
and are getting applied in diverse areas such as deep generative modeling~\cite{rezende14},
image segmentation~\cite{sohn15}, clustering~\cite{dilokthanakul17}, and future prediction from images~\cite{walker16}.

This paper discusses unsupervised learning with VAE which pipes an inference model
$q(\bm{z}\,\vert\,\bm{x})$ with a generative model $p(\bm{x}\,\vert\,\bm{z})$, 
where $\bm{x}$ and $\bm{z}$ are observed and latent variables, respectively.
A simple parameter-free prior $p(\bm{z})$ combined with $p(\bm{x}\,\vert\,\bm{z})$
parameterized by a deep neural network results in arbitrarily flexible representations.
However, its (very complex) posterior $p(\bm{z}\,\vert\,\bm{x})$ must be within the
representation power of the inference machine $q(\bm{z}\,\vert\,\bm{x})$,
so that the variational bound is tight and variational learning is effective.

In the original VAE~\cite{vae}, $q(\bm{z}\,\vert\,\bm{x})$ obeys a
Gaussian distribution with a diagonal covariance matrix. This is a very simplified assumption, because Gaussian is
the maximum entropy (least informative) distribution with respect to prescribed
mean and variance and has one single mode, while human inference can be
ambiguous and can have a bounded support when we exclude very unlikely cases~\cite{ev}.

Many recent works try to tackle this limitation.
Jang et al.~\citeyearpar{jang17} extended VAE to effectively use a discrete latent $\bm{z}$ following a category distribution (e.g. Bernoulli distribution).
Kingma et al.~\citeyearpar{kingma14} extended the latent structure with a combination of continuous and discrete latent variables (class labels)
and applied the model into semi-supervised learning.
Similarly, Shu et al.~\citeyearpar{shu16} and Dilokthanakul et al.~\citeyearpar{dilokthanakul17} proposed to use a Gaussian mixture latent model in VAE.
Serban et al.~\citeyearpar{serban17} applied a piecewise constant distribution on $\bm{z}$.

This work contributes a new ingredient in VAE model construction. To tackle the difficulty
in dealing with complex probability density function (pdf) $p(\bm{z})$ $(\bm{z}\in\mathcal{Z})$,
we generate instead a semi-continuous $\bm{z}\in\mathcal{Z}$, by first discretizing
the support $\mathcal{Z}$ into a grid, then drawing a discrete sample $\bm{y}$ based on the corresponding
probability mass function (pmf), and then reconstruct $\bm{y}$ into $\bm{z}\in\mathcal{Z}$.
This coarse grain (CG) technique can help apply \emph{any} pdf into VAE.
Hence we apply a bounded polynomial exponential family (BPEF)
as the underlying $p(\bm{z})$, which is a \emph{universal} pdf generator.
This fits in the spirit of neural networks because the prior and posterior
are not hand-crafted but \emph{learned by themselves}.

This contribution blends theoretical insights with empirical developments.
We present CG, BPEF, information monotonicity, etc., that are useful ingredients for general VAE modeling.
Notably, we present a novel application scenario with new analysis on the Gumbel softmax trick~\cite{jang17,maddison16}.
We assemble these components into a machine CG-BPEF-VAE and present empirical results on 
unsupervised density estimation, showing improvements over vanilla VAE~\cite{vae} and category VAE~\cite{jang17}.
We present a novel perspective with theoretical analysis of VAE learning,
with guaranteed bounds derived from information geometry~\cite{amari16}.

This paper is organized as follows.
\Cref{sec:vae} reviews the basics of VAE.
\Cref{sec:cgvae} introduces CG-VAE and its implementation CG-BPEF-VAE.
\Cref{sec:exp} performs an empirical study on two different datasets.
\Cref{sec:ig} gives a theoretical analysis on VAE learning.
\Cref{sec:con} states our concluding remarks.

\section{Prerequisites: Variational Autoencoders}\label{sec:vae}

This section covers the basics from a brief introduction
of variational Bayes to previous works on VAE.
A generative model can be specified by a joint distribution
between the observables $\bm{x}$ and the hidden variables $\bm{z}$, that is,
$p(\bm{x},\bm{z}\,\vert\,\bm\theta)
=p(\bm{z}\,\vert\,\bm\theta_z)p(\bm{x}\,\vert\,\bm{z},\bm\theta_{x|z})$
where $\bm\theta=(\bm\theta_z,\bm\theta_{x|z})$. 
By Jensen's inequality,
\begin{align}\label{eq:bound}
&-\log p(\bm{x}\,\vert\,\bm\theta)
=
-\log \int q(\bm{z}\,\vert\,\bm{x},\bm\varphi)
\frac{p(\bm{x},\bm{z}\,\vert\,\bm\theta)}{q(\bm{z}\,\vert\,\bm{x},\bm\varphi)} d\bm{z}
\nonumber\\
&\le
\int q(\bm{z}\,\vert\,\bm{x},\bm\varphi)
\log \frac{q(\bm{z}\,\vert\,\bm{x},\bm\varphi)}{p(\bm{x},\bm{z}\,\vert\,\bm\theta)}d\bm{z}
\left(\defeq\mathcal{L}(\bm\theta,\bm\varphi)\right),
\end{align}
for any $q(\bm{z}\,\vert\,\bm{x},\bm\varphi)$.
The upper bound $\mathcal{L}(\bm\theta,\bm\varphi)$
on the RHS is known as the ``variational free energy''. We have
\begin{align*}
&\mathcal{L}(\bm\theta,\bm\varphi) =
\underbrace{\KL(q(\bm{z}\,\vert\,\bm{x},\bm\varphi) : p(\bm{z}\,\vert\,\bm\theta_z))}_{term_1}\nonumber\\
&\hspace{5em}\underbrace{-\int q(\bm{z}\,\vert\,\bm{x},\bm\varphi) \log{p}(\bm{x}\,\vert\,\bm{z},\bm\theta_{x|z}) d\bm{z}}_{term_2},
\end{align*}
where $\KL(\cdot:\cdot)$ denotes the Kullback-Leibler (KL) divergence.
(We will use $term_1$ and $term_2$ as short-hands for the two terms 
whose sum is $\mathcal{L}(\bm\theta,\bm\varphi)$.
One has to remember that they are functions of $\bm\theta$ and $\bm\varphi$.)
We therefore minimize the free energy 
with respect to both $\bm\theta$ and $\bm\varphi$ so as to minimize
$-\log{p}(\bm{x}\,\vert\,\bm\theta)$. The gap of the bound in \cref{eq:bound} is
$
\mathcal{L}(\bm\theta,\bm\varphi) - (-\log p(\bm{x}\,\vert\,\bm\theta)) =
\KL\left(q(\bm{z}\,\vert\,\bm{x}, \bm\varphi) : p(\bm{z}\,\vert\,\bm{x}, \bm\theta)\right)
$,
which can be small as long as the parameter manifold of $q(\bm{z}\,\vert\,\bm{x}, \bm\varphi)$ 
(e.g. constructed based on the mean field technique, \citealt{jordan99})
encompasses a good estimation of the true posterior.

VAE~\cite{vae} assume the following generative process.  The prior
$p(\bm{z}\,\vert\,\bm\theta_z)=G(\bm{z}\,\vert\,\bm0,\bm{I})$ is parameter free,
where $G(\cdot\,\vert\,\bm\mu,\bm\Sigma)$ denotes a Gaussian distribution with
mean $\bm\mu$ and covariance matrix $\bm\Sigma$. 
Denote $\dim\bm{x}=D$ and $\dim\bm{z}=d$. The conditional mapping
$p(\bm{x}\,\vert\,\bm{z},\bm\theta) = 
\prod_{i=1}^D p\left(x_i\,\vert\, f(\bm{z}, \bm\theta) \right)$
is parametrized by a neural network $f(\bm{z}, \bm\theta)$ with input $\bm{z}$ and parameters $\bm\theta$.
For binary $\bm{x}$, $p(x_i\,\vert\,\cdot)$ is a Bernoulli distribution;
for continuous $\bm{x}$, $p(x_i\,\vert\,\cdot)$ can be univariate Gaussian.
This gives a very flexible $p(\bm{x}\,\vert\,\bm\theta)$ to adapt the complex data manifold.

In this case, it is hard to select the parameter form of
$q(\bm{z}\,\vert\,\bm{x},\bm\varphi)$, as the posterior 
$p(\bm{z}\,\vert\,\bm{x},\bm\varphi)$ has no closed form solution.
VAE borrows again the representation power of neural networks and lets
$q(\bm{z}\,\vert\,\bm{x},\bm\varphi)=
G(\bm{z}\,\vert\,\bm\mu(\bm{x}, \bm\varphi),\,\diag(\bm\lambda(\bm{x}, \bm\varphi)))$,
where $\bm\mu(\bm{x}, \bm\varphi)$ and
$\bm\lambda(\bm{x}, \bm\varphi)$ are both neural networks with input $\bm{x}$ and 
parameters $\bm\varphi$, and $\diag(\cdot)$ means a diagonal matrix constructed with a
given diagonal vector. The assumption of a diagonal covariance is for reducing
the network size so as to be efficient and to control overfitting.

Since the KL of Gaussians is available in closed form, $term_1$ has an analytical solution.
In order to solve the integration in $term_2$, VAE employs a \emph{reparameterization trick}.
It draws $L$ i.i.d. samples
$\bm\epsilon^1,\cdots,\bm\epsilon^L\sim{}G(\bm\epsilon\,\vert\,\bm0,\bm{I})$,
where $\bm{I}$ is the identity matrix.
Let $\bm{z}^l=\bm\mu(\bm{x}, \bm\varphi) + \bm\lambda(\bm{x},\bm\varphi)\circ\bm\epsilon^l$,
where ``$\circ$'' denotes element-wise product.
Then $\bm{z}^l\sim{G}(\bm{z}\,\vert\,\bm\mu(\bm{x}, \bm\varphi), 
\diag(\bm\lambda(\bm{x}, \bm\varphi)))$. Hence
$$term_2 \approx -\frac{1}{L}\sum_{l=1}^{L}
\log p\left(\bm{x}\,\vert\, \bm\mu(\bm{x}, \bm\varphi) +
\bm\lambda(\bm{x}, \bm\varphi)\circ\bm\epsilon^l, \bm\theta \right).$$
This trick allows error to backpropagate through the random mapping $(\bm\mu,\bm\lambda)\leadsto\bm{z}$.

Then $\mathcal{L}(\bm\theta,\bm\varphi) = term_1 + term_2$
can be expressed as simple arithmetic operations of
the outputs of the hidden layer and the last layer. It can therefore be
optimized e.g. with stochastic gradient descent. The optimization
technique is called stochastic gradient variational Bayes (SGVB).
The resulting architecture is presented in \cref{fig:vae}.

\begin{figure*}[t]
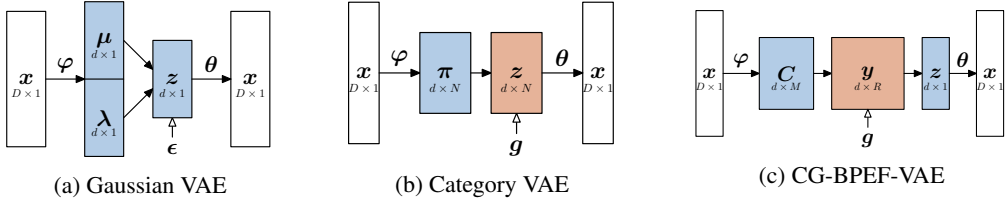

\centering
\begin{subfigure}[m]{.3\textwidth}
\centering\includegraphics[width=.8\textwidth]{vae.1}
\caption{Gaussian VAE}\label{fig:vae}
\end{subfigure}
\begin{subfigure}[m]{.3\textwidth}
\centering\includegraphics[width=.8\textwidth]{vae.2}
\caption{Category VAE}\label{fig:cvae}
\end{subfigure}
\begin{subfigure}[m]{.35\textwidth}
\centering\includegraphics[width=.8\textwidth]{vae.3}
\caption{CG-BPEF-VAE}\label{fig:cgvae}
\end{subfigure}
\caption{Architecture of three different VAEs. Blue indicates that
the corresponding variable is continuous; red means discrete.
CG-BPEF-VAE features a sandwich structure.}
\end{figure*}

\section{CG-BPEF-VAE}\label{sec:cgvae}

We would like to extend VAE to incorporate a \emph{general} inference process,
where the model can learn by itself a proper $p(\bm{z}\,\vert\,\bm{x},\bm\varphi)$
within a flexible family of distributions, which is not limited to Gaussian or category distributions
and can capture higher order moments of the posterior.
We will therefore derive in this section a variation of VAE called
CG-BPEF-VAE for Coarse-Grained Bounded Polynomial Exponential Family VAE.

\subsection{Bounded Polynomial Exponential Family}\label{subsec:bpef}

We try to model the latent $\bm{z}$ with a \emph{factorable} polynomial
exponential family (PEF) \citep{pef,patchpef} probability density function:
\begin{equation}\label{eq:bpef}
p(\bm{z}) = \prod_{j=1}^d \exp\left(\sum_{m=1}^{M} c_{jm} z_j^{m} -\psi(\bm{c}_j)\right),
\end{equation}
where $M$ is the polynomial order,
$\bm{C}=(c_{jm})_{d\times{M}}$ denotes the polynomial coefficients,
and $\psi$ is a convex cumulant generating function~\cite{amari16}.
This PEF family can be regarded as the most general parameterization,
because with large enough $M$ it can approximate \emph{arbitrary finely}
any given $p(\bm{z})$ satisfying weak regularity conditions~\citep{pef}.

Furthermore, we constrain $\bm{z}$ to have a \emph{bounded support} so that $\bm{z}\in[-1,1]^d$,
a hypercube. This gives $\bm{z}$ a focused density that is not wasted on unlikely cases,
which is in contrast to Gaussian distribution with non-zero probability on the whole real line.
This also allows one to easily explore \emph{extreme cases} by setting $z_j$ to $\pm1$ or beyond.

For example, if $M=2$, then the resulting $p(z_j)\propto\exp\left(c_{j1}z_j + c_{j2}z_j^2\right)$
includes the truncated Gaussian distribution (with one mode) as a special case when $c_{j2}<0$.
Moreover, the setting $c_{j2}\ge0$ encompasses more general cases and can have at most two modes.

The two important elements in constructing a VAE model are
\ding{192} the KL divergence between $q(\bm{z}\,\vert\,\bm{x},\bm\varphi)$
and $p(\bm{z}\,\vert\,\bm\theta_z)$ must have a closed form;
\ding{193} a random sample of $q(\bm{z}\,\vert\,\bm{x},\bm\varphi)$ can be
expressed as a simple function between its parameters 
and some parameter-free random variables.
Neither of these conditions are met for BPEF.
We will address these difficulties in the remainder of this section.

\subsection{Coarse Grain}

Our basic idea is to reduce the BPEF pdf into a discrete distribution,
then draw samples based on the pmf, then reconstruct the continuous sample.

We sample $R$ points uniformly on the interval $[-1,1]$:
$$\bm\zeta=\left( -1, -1+\frac{2}{R-1}, \cdots, 1-\frac{2}{R-1}, 1 \right)^\intercal,$$
where the $r$'th discrete value is $\zeta_r=\frac{2r-(R+1)}{R-1}$. 
For example, choosing $R=21$ results in a precision of 0.1.
In correspondence to these $R$ locations, we assume for the $j$'th latent dimension
a random $\bm{y}_j$ in $\Delta^{R-1}$, the $(R-1)$-dimensional probability simplex,
so that $\sum_{r=1}^R y_{jr}=1$, $\forall{r}$, $y_{jr}\ge0$.
This $y_{jr}$ means the likelihood for $z_j$ taking the value $\zeta_r$.
Intuitively, if we constrain $\bm{y}_j$ to be one-hot (with probability mass only on vertices of $\Delta^{R-1}$),
and let $P(y_{jr}=1)\propto\exp(\sum_{m=1}^M c_{jm} \zeta_r^m)$, then the expectation
$z_j=\sum_{r=1}^R y_{jr} \zeta_r\in[-1,1]$ will be distributed like the BPEF in \cref{eq:bpef}.

However, to apply the reparameterization trick, it is not known how to express a random one-hot sample $\bm{y}_j$
as a simple function of the activation probabilities. Nor does Dirichlet
distribution as a commonly-used density on $\Delta^{R-1}$ can do the trick.

This reparameterization problem of category distribution is studied recently~\cite{jang17,maddison16}
following earlier developments~\cite{kuzmin2005,maddison14} on applying extreme
value distributions~\cite{ev} to machine learning.
Based on these previous studies, we let $\bm{y}_j$ follow a \emph{Concrete distribution}~\cite{maddison16},
which is a continuous relaxation of the category distribution, with the key 
advantage that Concrete samples can be easily drawn to be applied to VAE.
Details are explained as follows.



The standard Gumbel distribution~\cite{gumbel54} is defined on the support 
$g\in\Re$ with the cumulative distribution function $P\left(g\le{}x\right) = e^{-e^{-x}}$.
Therefore Gumbel samples can be easily obtained by inversion sampling
$g=-\log(-\log{U})$, where $U$ is uniform on $(0,1)$.
Let $g_{jr}$ follows standard Gumbel distribution, then the random variable $\bm{y}_j\in\Delta^{R-1}$
defined by 
\begin{equation*}
y_{jr}=\frac{\exp\left((g_{jr}+\phi_{jr})/T\right)}{\sum_{r=1}^R \exp\left((g_{jr}+\phi_{jr})/T\right)}
\end{equation*}
is said to follow a Concrete distribution with location parameter $\bm\phi_{j}$ and temperature parameter $T$:
$\bm{y}_{j}\sim{}\mathrm{Con}(\bm\phi_j,T)$. This distribution has a closed-form probability density function
(see \citealt{maddison16}) and has the following fundamental property
\begin{align}\label{eq:gumbelmax}
\forall{r},\quad%
&P\left(\lim_{T\to0^+} y_{jr}=1\right)
= P\left(y_{jr} > y_{jo}, \forall{o\neq{r}}\right)\nonumber\\
&= {\exp(\phi_{jr})}/{\sum_{r=1}^R \exp(\phi_{jr})}
\;\left(\defeq\alpha_{jr}\right).
\end{align}
Basically, at the limit $T\to0^+$, the density will be pushed to the vertices
of $\Delta^{R-1}$, and Concrete random vectors $\bm{y}_{j}$
tend to be onehot, with activation probability of the $r$'th bit defined by $\alpha_{jr}$.
Hence it can be considered as a relaxation~\cite{maddison16} of the category distribution.
See \cref{fig:concrete} for an intuitive view of the Concrete distribution.
There are heavy volumes of densities around the vertices.

\begin{figure}[t]
\centering
\includegraphics[width=.8\textwidth]{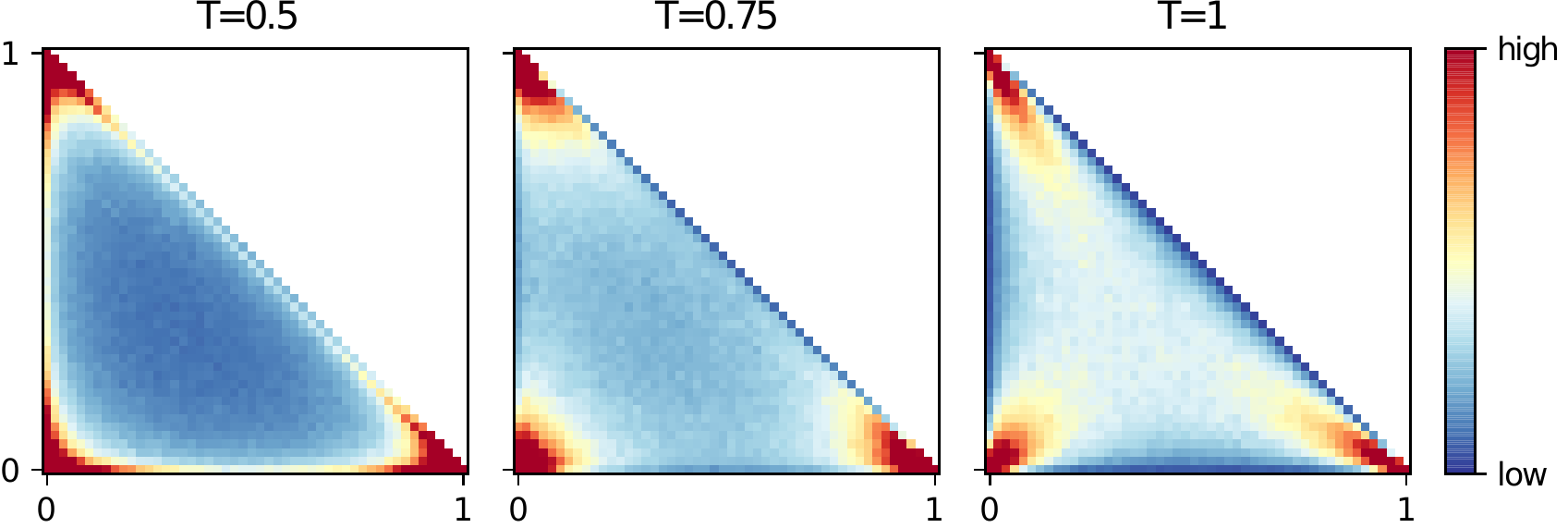}
\caption{Density distribution of $\mathrm{Con}((1/3,1/3,1/3),T)$ over $\Delta^2$.
The figure is generated by random sampling (see appendix for more figures and discussions).}\label{fig:concrete}
\end{figure}


In our case, let $\phi_{jr}=\sum_{m=1}^{M} c_{jm} \zeta_r^m$,
then the odds for $y_{jr}$ activated (i.e., the probability for $z_{j}$ taking the value $\zeta_r$)
will be proportional to $\exp\left(\sum_{m=1}^M c_{jm}\zeta_r^m\right)$ at the limit $T\to0^+$.
This provides a way to simulate the BPEF density.




\subsection{The Model}


Based on previous subsections, we assume the following generation process
\begin{align*}
&\alpha_{jr} = \sum_{m=1}^{M} a_{jm} \zeta_r^m,
\quad{}p(\bm{y}\,\vert\,\bm{a}) = \prod_{j=1}^d
\mathrm{Con}\left(\bm{y}_j\,\vert\,\bm{\alpha}_{j},T\right),\\
&{z}_{j}(\bm{a}) = \bm{y}_j^\intercal\bm\zeta,
\hspace{3em}p(\bm{x}\,\vert\,\bm{z},\bm\theta)
= \prod_{i=1}^D p\left(x_i\,\vert\,f\left(\bm{z}, \bm\theta\right)\right),
\end{align*}
where $\bm{A}=(a_{jm})_{d\times{}M}$ is the parameters of the prior\footnote{Strictly speaking the prior
distribution only contains hyper-parameters that are set a priori. Here the term ``prior''
is more like a prior structure with learned parameters.}, and $f$ is defined by a neural network.
One should always choose $M<R-1$, because the polynomial $\sum_{m=1}^{R-1} c_{jm} \zeta_{r}^m$ with $R-1$
free parameters can already represent any distribution in $\Delta^{R-1}$.
The setting $M\ge{}R-1$ makes the polynomial structure redundant.

The corresponding inference process is given by
\begin{align*}
\beta_{jr} &= \sum_{m=1}^M b_{jm}(\bm{x},\bm\varphi) \zeta_r^m,\\
q(\bm{y}\,\vert\,\bm{x},\bm\varphi) &= \prod_{j=1}^d \mathrm{Con}
\left( \bm{y}_j\,\vert\,\bm\beta_j, T \right),\\
z_j(\bm{x},\bm\varphi) &= \bm{y}_j^\intercal \bm\zeta,
\end{align*}
where $\bm{B}(\bm{x},\bm\varphi)=(b_{jm}(\bm{x},\bm\varphi))_{d\times{}M}$ is defined by a neural network.
By Monte Carlo integration, it is straightforward that 
\begin{align*}
&term_2 \approx \frac{1}{L}\sum_{l=1}^L \sum_{i=1}^D
\log p\left(x_i\,\vert\, f\left( \sum_{r=1}^R  y_{\bullet{}r}^l\zeta_r, \bm\theta \right)\right),\nonumber\\
&y_{jr}^l=
\frac{\exp\left((g_{jr}^l +\sum_{m=1}^{M} b_{jm}(\bm{x},\bm\varphi) \zeta_r^m)/T\right)}
{\sum_{r=1}^R \exp\left((g_{jr}^l +\sum_{m=1}^{M} b_{jm}(\bm{x},\bm\varphi) \zeta_r^m)/T\right)},
\end{align*}
where $(g_{jr}^l)$ is a 3D tensor of independent Gumbel variables,
and the approximation becomes accurate when $L\to\infty$.

For simplicity, we assume $T$ to be the same scalar during generation and inference.
We adopt a simple annealing process of $T$, starting from $T_{\max}$, exponentially
decaying to $T_{\min}$ in the first half of training epochs, then keeping $T_{\min}$.
The study~\cite{jang17} implies that $T_{\min}=0.5\sim1$ could be small enough to make the
Concrete distribution approximate well a category distribution. The setting of
$T_{\min}$ will affect the computation of $term_1$, which will be explained in the following subsection.



\subsection{Information Mononicity}\label{subsec:kl}

We need to compute $term_1$ which is the KL divergence between the posterior 
$p(\bm{z}\,\vert\,\bm{x},\bm\varphi)$ the prior $p(\bm{z}\,\vert\,\bm{a})$. This is
the most complex part because these pdfs are not in closed form. However, we know
that as $T\to0^+$ they converge to categories distributions over $R$ evenly spanned positions 
on  $[-1,1]$ (the vector $\bm\zeta$).
Therefore we approximate $term_1$ with the KL divergence between the corresponding category distributions, that is,
\begin{align}\label{eq:monotonicity}
term_1 \approx&
\sum_{j=1}^d \sum_{r=1}^R \Bigg[
\frac{\exp(\beta_{jr})}{\sum_{r=1}^R \exp(\beta_{jr})}\nonumber\\
&
\times\log\frac{ \exp(\beta_{jr})/(\sum_{r=1}^R \exp(\beta_{jr}))}
{\exp(\alpha_{jr})/(\sum_{r=1}^R \exp(\alpha_{jr}))} \Bigg]\nonumber\\
=&\sum_{j=1}^d \bigg[
\frac{\sum_{r=1}^R\exp( \beta_{jr} ) (\beta_{jr}-\alpha_{jr}) }{\sum_{r=1}^R\exp(\beta_{jr})}
\nonumber\\
&+\log\sum_{r=1}^R\exp(\alpha_{jr} )
 -\log\sum_{r=1}^R\exp(\beta_{jr} ) \bigg].
\end{align}

In the rest of this subsection we give theoretical and empirical justifications
of this approximation.
KL divergence belongs to Csisz\'ar's $f$-divergence family and therefore
satisfy the well-known information monotonicity~\cite{amari16}. 
Basically, the support $V$ can be partitioned into subregions $\{V_r\}$ with zero
volume overlap, so that $V=\uplus{}V_r$. Denote by $p_1(V_r) = \int_{x\in{V}_r} p_1(x) \dx$
the probability mass of $V_r$, then $\sum_{r}p_1(V_r)=1$ and the pmf $\{p_1(V_r)\}$
is a coarse grained version of $p_1(x)$. The information monotonicity principle states that
$\KL(p_1:p_2) \ge \sum_{r} p_1(V_r) \log\frac{p_1(V_r)}{p_2(V_r)}$.
See~\cite{klmm} for an analysis. Based on this principle, we have the following result.
\begin{theorem}\label{thm:mono}
\begin{align*}
\text{\ding{192}}\;\KL(q(\bm{y}\,\vert\,\bm{x},\bm\varphi):p(\bm{y}\,\vert\,\bm{a})) &\ge
\KL(q(\bm{z}\,\vert\,\bm{x},\bm\varphi):p(\bm{z}\,\vert\,\bm{a}))\\
&= term_1;
\end{align*}
\ding{193} $\KL(q(\bm{y}\,\vert\,\bm{x},\bm\varphi):p(\bm{y}\,\vert\,\bm{a}))$ is also
lower bounded by the discrete KL given by the right hand side of \cref{eq:monotonicity}.
\end{theorem}

By \cref{thm:mono}, the KL between two Concrete distributions are lower bounded by
\ding{192} KL between the dimension reduced $\bm{z}$ (the exact value of $term_1$);
\ding{193} KL between the corresponding category distributions (our approximation of $term_1$).
If one uses Concrete latent variable and uses the category KL as $term_1$
(e.g. in Category VAE, see \cref{fig:cvae}), this is equivalent to minimizing
a \emph{lower bound} of the free energy, which is not ideal because such learning
has less control over the free energy. In contrast, CG-BPEF-VAE has a 
reconstruction layer $\bm{y}\to\bm{z}$ (see \cref{fig:cgvae}), which reduces
the number of dimensions by a factor of $R$ (e.g. in our experiments $R\approx100$).
By \cref{thm:mono} \ding{192}, this effectively reduces the KL divergence between
the latent posterior and the latent prior. Intuitively, we can expect $term_1$
to be much smaller than $\KL(q(\bm{y}\,\vert\,\bm{x},\bm\varphi):p(\bm{y}\,\vert\,\bm{a}))$
and by minimizing the category KL, we have more faith to bring down $term_1$ 
rather than $\KL(q(\bm{y}\,\vert\,\bm{x},\bm\varphi):p(\bm{y}\,\vert\,\bm{a}))$.

\emph{How good is our approximation} in \cref{eq:monotonicity}? Unfortunately we do not
have theoretically guaranteed bounds. Therefore we fall back to an empirical study.
We generate category samples $\bm\alpha\in\Delta^{99}$, then generate the corresponding Gumbel
distribution $\bm{y}$, then reduce the dimensionality by $z=\bm{y}^\intercal \bm\zeta$.
\Cref{fig:bound} shows the $\KL(\bm\alpha:\mathrm{Uniform})$ (our approximation)
and $\KL(p(\bm{z}):\mathrm{Uniform})$ (the true latent KL).
We repeat 100 experiments for each of two different $\bm\alpha$ generator:
a high entropy uniform generator over $\Delta^{99}$,
and a low entropy generator based on a Dirichlet distribution with shape parameter $\alpha=0.5$.
(In practice we expect a low entropy posterior which is close to the latter case).
The results suggest that our approximation is roughly an upper bound of the true KL divergence
between latent distributions on small temperatures. Therefore we can expect
that minimizing $\mathcal{L}(\bm\theta,\bm\varphi)$ based on
\cref{eq:monotonicity} will bring down the free energy.
See the appendix for more empirical study.
A theoretical analysis is left to future work.





Essentially $term_1$ serves as a regularizor, constraining $p(\bm{z}\,\vert\,\bm{x},\bm\varphi)$ 
to have enough entropy to respect a common $p(\bm{z})$ that does not vary with different samples.
An approximated $term_1$ is acceptable in many cases, because
one can add a regularization strength parameter to tune the model (e.g. based on validation).


\begin{figure}[t]
\centering
\includegraphics[width=.75\textwidth]{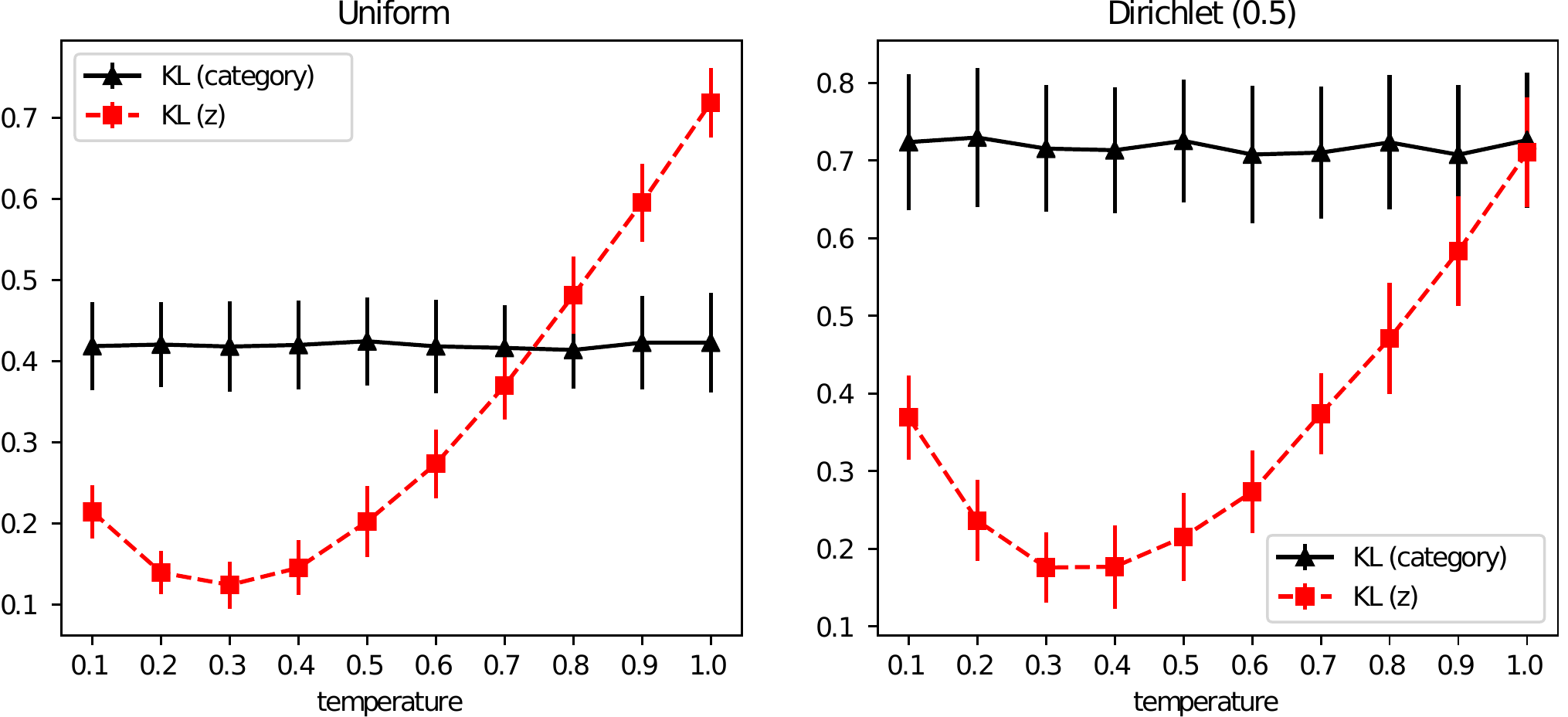}
\caption{KL divergence between category distributions v.s. KL divergence
between corresponding Gumbel distributions ($\bm{y}$) after dimensionality reduction ($\bm{z}$).
The figures show mean$\pm$standard deviation of KL. Both of KL divergences are computed by
discretizing $(-1,1)$ into $R=100$ intervals.}\label{fig:bound}
\end{figure}

\section{Experimental Results}\label{sec:exp}

\begin{table*}[t]
\scriptsize
\centering
\caption{Training and testing errors (estimated variational bound) on the MNIST dataset and the corresponding model configuration.
``iconv'' consists 5 convolutional layers; ``oconv'' consists of one
RELU layer and 4 transposed convolutional layers.}\label{tbl:mnist}
\begin{tabular}{l|r@{~~}l|ccccccc}
\hline
&\multicolumn{2}{c|}{Model}    & $term_2$ & $\mathcal{L}$ & $\gamma$ & network shape & $C$ & $M$ & $R$ \\
\hline
\multirow{5}{*}{$L=1$} &
Gauss-VAE
 &       &  80.2 & 101.3 & 1e-3 & 784-400-400-20-400-400-784 & $-$ & $-$ & $-$\\\cline{2-10}
&\multirow{2}{*}{Cat-VAE}
 & T=0.5 & 86.5 & 105.8 & 1e-3 & 784-400-400-20-400-400-784 &  5  & $-$ & $-$\\
&& T=0.8 & 84.2 & 101.5 & 1e-3 & 784-400-400-20-400-400-784 & 10 & $-$ & $-$ \\\cline{2-10}
&\multirow{2}{*}{CG-BPEF-VAE}
 & T=0.5 & 77.8 & $\bm{97.4}$ & 1e-3 & 784-400-400-20-400-400-784 & $-$ &  5 & 101 \\
&& T=0.8 & 75.9 & $\bm{92.8}$ & 1e-3 & 784-400-400-30-400-400-784 & $-$ &  5 & 101 \\\hline
\multirow{5}{*}{$L=10$} &
Gauss-VAE
 && 78.5 & 100.0 & 1e-3 & 784-400-400-20-400-400-784 & $-$ & $-$ & $-$\\\cline{2-10}
&\multirow{2}{*}{Cat-VAE}
 & T=0.5 & 78.8 & 98.2 & 1e-3 & 784-400-400-20-400-400-784 & 15 & $-$ & $-$ \\
&& T=0.8 & 76.9 & 94.2 & 1e-3 & 784-400-400-20-400-400-784 & 20 & $-$ & $-$ \\\cline{2-10}
&\multirow{2}{*}{CG-BPEF-VAE}
 & T=0.5 & 74.8 & $\bm{94.7}$ & 1e-3 & 784-400-400-60-400-400-784 & $-$ & 15 & 101 \\
&& T=0.8 & 73.2 & $\bm{90.4}$ & 1e-3 & 784-400-400-40-400-400-784 & $-$ & 5 & 101 \\
\hline
\end{tabular}
\subcaption{MNIST}
\begin{tabular}{l|r@{~~}l|ccccccc}
\hline
&\multicolumn{2}{c|}{Model}     & $term_2$ & $\mathcal{L}$ & $\gamma$ & network shape & $C$ & $M$ & $R$ \\
\hline
\multirow{5}{*}{$L=1$} &
Gauss-VAE
 &       &  616.3 & 624.7 & 1e-3 & 1024-iconv-128-30-128-oconv-1024 & $-$ & $-$ & $-$\\\cline{2-10}
&\multirow{2}{*}{Cat-VAE}
 & T=0.5 & 619.5 & 626.2 & 1e-3 & 1024-iconv-128-20-128-oconv-1024 & 10 & $-$ & $-$\\
&& T=0.8 & 617.6 & 623.7 & 1e-3 & 1024-iconv-128-20-128-oconv-1024 & 10 & $-$ & $-$\\\cline{2-10}
&\multirow{2}{*}{CG-BPEF-VAE}
 & T=0.5 & 615.4 & $\bm{622.7}$ & 1e-3 & 1024-iconv-128-20-128-oconv-1024 & $-$ & 15 & 101 \\
&& T=0.8 & 613.4 & $\bm{620.3}$ & 1e-3 & 1024-iconv-128-20-128-oconv-1024 & $-$ &  5 & 101 \\
\hline
\multirow{5}{*}{$L=10$} &
Gauss-VAE
& & 615.7 & 624.4 & 1e-3 & 1024-iconv-128-10-128-oconv-1024 & $-$ & $-$ & $-$\\\cline{2-10}
&\multirow{2}{*}{Cat-VAE}
 & T=0.5 & 616.6 & 623.2 & 1e-3 & 1024-iconv-128-20-128-oconv-1024 & 15 & $-$ & $-$ \\
&& T=0.8 & 615.1 & 620.4 & 1e-3 & 1024-iconv-128-20-128-oconv-1024 & 20 & $-$ & $-$ \\\cline{2-10}
&\multirow{2}{*}{CG-BPEF-VAE}
 & T=0.5 & 613.4 & $\bm{621.6}$ & 1e-3 & 1024-iconv-128-20-128-oconv-1024 & $-$ & 10 & 101 \\
&& T=0.8 & 611.7 & $\bm{619.0}$ & 1e-3 & 1024-iconv-128-30-128-oconv-1024 & $-$ & 5 & 101 \\
\hline
\end{tabular}
\subcaption{SVHN}
\end{table*}

We implemented the proposed method using TensorFlow~\cite{tensorflow}
and tested it on two different datasets. The MNIST dataset~\cite{mnist} consists of 70,000 gray scale images
of hand-written digits, each of size $28\times28$. The training/validation/testing
sets are split according to the ratio $11:1:2$. The SVHN dataset~\cite{svhn} has around
100,000 gray-scale pictures (for simplicity the original $32\times32\times3$ RGB images are reduced into 
$32\times32\times1$ by averaging the 3 channels) of door numbers with a train/valid/test split of $10:1:3.5$.
These pictures are centered by cropping from real street view images.

We only investigate unsupervised density estimation. It is nevertheless
meaningful to have unsupervised VAE results on the selected datasets for future references.
We compare the proposed CG-BPEF-VAE with vanilla VAE (Gauss-VAE) and Category VAE (Cat-VAE)~\cite{jang17}.
For MNIST, the candidate network shapes are
784-400-(10,20,$\cdots$,80)-400-784 and 784-400-400-(10,20,$\cdots$,80)-400-400-784,
equipped with densely connected layers and RELU activations~\cite{nair10}.
For SVHN, the encoder network has 5 convolutional layers with fixed size,
reducing the images into a 128-dimensional feature space, and a bottleneck
layer of size (10,20,30,40,50) (5 different configurations). The decoder network
has one RELU layer of size 128, followed by 4 transposed convolutional layers
with fixed size. See the appendix for the detailed configurations.

The learning rate is $\gamma\in\{10^{-4}, 5\times{}10^{-4}, 10^{-3}, 5\times10^{-3}\}$.
For Cat-VAE and CG-BPEF-VAE, the initial and final temperature are 
$T_{\max}\in\{1,T_{\min}\}$ and $T_{\min}\in\{0.5,0.8\}$, respectively,
 with a simple exponential annealing scheme.
For Cat-VAE, the number of categories is $C\in\{5,10,15,20\}$.
For CG-BPEF-VAE, we set the polynomial order $M\in\{5,10,15,20\}$, and the precision $R\in\{51, 101\}$.
The mini-batch size is fixed to 100. The maximum number of mini-batch iterations is 10,000.
For all methods we adopt the Adam optimizer~\cite{adam} and the
Xavier initialization~\cite{xavier}, which are commonly recognized to bring improvements.

The performance is measured by the per-sample average free energy $\mathcal{L}(\bm\varphi,\bm\theta)$.
The best model with the smallest validated $\mathcal{L}$ is selected.
Then the we report its $\mathcal{L}$ on the testing set, along with the reconstruction error $term_2$
so that one can tell its trade-off between model complexity ($term_1$) and fitness to the data ($term_2$).
See \cref{tbl:mnist} for the results
on two different latent sample size $L$ and two different temperatures $T_{\min}$.

We clearly see that CG-BPEF-VAE shows the best results.
Essentially, Gauss-VAE can be considered as a special case of CG-BEPF-VAE
when $M=2$ therefore cannot model higher order moments. Cat-VAE has neither a polynomial
exponential structure to regulate the discrete variables, nor
a dimensionality reduction layer to reduce the free energy.
The good results of CG-BPEF-VAE are expected.

Notice that as we increase the final temperature $T_{\min}$, both Cat-VAE and CG-BPEF-VAE
will show ``better'' results. However, the estimation of the free energy 
will become more and more inaccurate especially for Cat-VAE, whose estimation
is a \emph{lower bound} of the actual free energy by \cref{thm:mono} (2).
In high temperature, the free energy can be well above its reported $\mathcal{L}$.
In contrast, for CG-BPEF-VAE, its estimated $\mathcal{L}$ is an \emph{empirical upper bound}
of the free energy, as long as $T_{\min}$ is set reasonably small ($T=0.5\sim1$, see \cref{fig:bound}).

All models prefer deep architectures over shallow ones. There is a significant improvement
of Cat-VAE and CG-BPEF-VAE when $L$ is increased from 1 to 10, when Cat-VAE
starts to prefer larger category numbers. A large sample size $L$ is required to model complex
multimodal distributions and is recommended for Cat-VAE and CG-BPEF-VAE.
As the size of the decoder network scales linearly with $L$, one will face
significantly higher computation cost during increasing $L$.

As compared to MNIST, SVHN is more difficult to get improved over
the baseline results by Gauss-VAE, because its data manifold
is much more complex. One has to incorporate supervised
information~\cite{kingma14} to achieve better results.

Cat-VAE and CG-BPEF-VAE are more computational costly than Gauss-VAE.
In Cat-VAE, the tensor $\bm{z}$ has a size of 
$\mathtt{batch~size}\times{L}\times{d}\times{}C$.
In CG-BPEF-VAE, the tensor $\bm{y}$
have a size of $\mathtt{batch~size}\times{L}\times{d}\times{}R$, although
this is immediately reduced to $\mathtt{batch~size}\times{L}\times{d}$ by the mapping $\bm{y}\to\bm{z}$.
A high precision CG-BPEF-VAE or a Cat-VAE with a large category number will multiply the computational time.
Our implementation is available at \url{https://github.com/sunk/cgvae}.

\section{Information Geometry of VAE}\label{sec:ig}

This is a relatively separate section. We present a geometric theory which can
be useful to uncover the intrinsics of \emph{general} VAE modeling not limited
to the proposed CG-BPEF-VAE, so that one can architect useful VAE models not
only based on variational inference, but also along another geometric axis. 
We also use this geometry to discuss advantages of the proposed CG-BPEF-VAE.

Notice, this geometry
is not about the input feature space or the latent space (space of $\bm{x}$
and $\bm{z}$), but about the models (space of $\bm\theta$ and $\bm\varphi$)
or information geometry~\cite{amari16}.

We will consider the cost function $\calL(\bm\theta,\bm\varphi)$
averaged with respect to i.i.d. observations $\{\bm{x}^k\}_{k=1}^n$. $term_1$ is the average
KL divergence between $q(\bm{z}\,\vert\,\bm{x}^k)$ and $p(\bm{z})$. 
Assume that both $p(\bm{z})$ and $q(\bm{z}\,\vert\,\bm{x}^k)$ are in the
same exponential family $\mathcal{M}(\bm\varphi)$
so that $p(\bm{z})=\exp(\bm{t}^\intercal(\bm{z})\bm\varphi^z - F(\bm\varphi^z))$
and $q(\bm{z}\,\vert\,\bm{x}^k)= \exp(\bm{t}^\intercal(\bm{z})\bm\varphi^k - F(\bm\varphi^k))$,
where $\bm{t}(\bm{z})$ is a vector of sufficient statistics
(for example in CG-BPEF-VAE, $\bm{t}(\bm{z})=(z,z^2,z^3,\cdots)$), 
and $F(\bm\varphi)$ is a convex cumulant generating function\footnote{
In this section, we will denote $p(\bm{z}\,\vert\,\bm\varphi^z)$ instead of 
$p(\bm{z}\,\vert\,\bm\theta^z)$ (as in previous sections) to emphasize that 
$p(\bm{z})$ is in the same statistical manifold with $q(\bm{z}\,\vert\,\bm{x}^k)$.}.
This $\mathcal{M}(\bm\varphi)$ is a \emph{statistical manifold}, i.e., space
of probability distributions where $\bm\varphi$ serves as a coordinate system.
The dual parameters~\cite{amari16} of $\mathcal{M}(\bm\varphi)$, which form
another coordinate system, are defined by the moments $\bm\eta=E(\bm{t}(\bm{z}))=\int p(\bm{z})\bm{t}(\bm{z})\mathrm{d}\bm{z}$.
These two coordinate systems can be transformed back and forth by the Legendre transformations
$\bm\eta=F'(\bm\varphi)$, $\bm\varphi=\calI'(\bm\eta)$, where $\calI$ is Shannon's information (negative entropy).

By straightforward derivations, 
$$
term_1=\frac{1}{n}\sum_{k=1}^n\left[ \calI(\bm\eta^k) - 
(\bm\eta^k)^\intercal \bm\varphi^z \right] + F(\bm\varphi^z).
$$
Notice that the prior $p(\bm{z})$ only appears in $term_1$ but not in $term_2$. We therefore consider a free
$p(\bm{z})$ which minimizes $term_1$ with $\{\bm\varphi^k\}_{k=1}^n$ fixed. We have
$$
\frac{\partial{}term_1}{\partial\bm\varphi^z}
= -\frac{1}{n}\sum_{k=1}^n\bm\eta^k + \frac{\partial F(\bm\varphi^z)}{\partial\bm\varphi^z}
= \bm\eta^z -\frac{1}{n}\sum_{k=1}^n\bm\eta^k.
$$
Therefore the optimal $(\bm\eta^z)^\star=\frac{1}{n}\sum_{k=1}^n\bm\eta^k$
is the Bregman centroid~\cite{nielsen09} of $\{\bm\varphi^k\}_{k=1}^n$.
Geometrically, $term_1$ is the average divergence between $\bm\varphi^k$
and the Bregman centroid and therefore measures the $n$-body compactness
of $\{\bm\varphi^k\}_{k=1}^n$. We can therefore have a lower bound of $term_1$.
\begin{theorem}\label{thm:lb}
Given $q(\bm{z}\,\vert\,\bm{x}^k)$ in an exponential family $\mathcal{M}(\bm\varphi)$,
if $p(\bm{z})$ is in the same exponential family, then
\begin{equation}\label{eq:parabound}
term_1 \ge
\frac{1}{n}\sum_{k=1}^n\calI\left(\bm\eta^k\right) -
\calI\left(\frac{1}{n}\sum_{k=1}^n\bm\eta^k\right)\ge0 ,
\end{equation}
where the first ``='' holds if and only if $\bm\eta^z=\frac{1}{n}\sum_{k=1}^n\bm\eta^k$.
If $p(\bm{z})$ is non-parametric (not constrained by any parametric structure), then
\begin{equation}\label{eq:nonparabound}
term_1 \ge \frac{1}{n}\sum_{k=1}^n \calI(\bm\eta^k) - \calI(m) \ge 0,
\end{equation}
where $m(\bm{z})=\frac{1}{n}\sum_{k=1}^n q(\bm{z}\,\vert\,\bm{x}^k)$
is a mixture model which is outside $\mathcal{M}(\bm\varphi)$.
\end{theorem}
Comparatively, the non-parametric lower bound \cref{eq:nonparabound} is smaller than the parametric bound \cref{eq:parabound}.
However it needs to compute the entropy of mixture models~\cite{klmm}, which does not have an analytic solution.
Essentially, $term_1$ is related to the convexity of Shannon information.
In standard VAE, $p(\bm{z})$ is fixed to the standard Gaussian distribution,
which is not guaranteed to be the Bregman centroid, and does not activate
the lower bound in \cref{thm:lb}. In CG-BPEF-VAE, $term_1$ is closer to this 
bound because $p(\bm{z})$ is set free in our modeling. This hints that
as a future work one can directly replace $term_1$ with the lower bound stated in \cref{thm:lb}
to avoid the model selection of $p(\bm{z})$ and to achieve better performance.

Let $\bm\mu^k=\bm\mu(\bm{x}^k,\bm\varphi)$ and $\bm{V}^k=\bm{V}(\bm{x}^k,\bm\varphi)\succ0$
be the mean and covariance matrix of $q(\bm{z}\,\vert\,\bm{x}^k,\bm\varphi)$, respectively.
A Taylor expansion of $\log p(\bm{x}^k\,\vert\,\bm{z},\bm\theta)$ at $\bm{z}=\bm\mu^k$ gives 
\begin{small}
\begin{align*}
&term_2
\approx \frac{1}{n} \sum_{k=1}^n
\int q(\bm{z}\,\vert\,\bm{x}^k,\bm\varphi)
\bigg[ -\log p(\bm{x}^k\,\vert\,\bm\mu^k,\bm\theta) \nonumber\\
&
- (\bm{z}-\bm\mu^k)^\intercal
\frac{\log p(\bm{x}^k\,\vert\,\bm{z},\bm\theta)}{\partial\bm{z}}
+\frac{1}{2}(\bm{z}-\bm\mu^k)^\intercal \fim_{\bm\theta}(\bm\mu^k) (\bm{z}-\bm\mu^k)
\bigg] d\bm{z}\nonumber\\
&=
\frac{1}{n} \sum_{k=1}^n \bigg[
- \log p(\bm{x}^k\,\vert\,\bm\mu^k, \bm\theta)
+ \frac{1}{2}tr\left(\fim_{\bm\theta}(\bm\mu^k) \bm{V}^k\right) \bigg],
\end{align*}
\end{small}
where
$$
\fim_{\bm\theta}(\bm\mu^k)= - \frac{\partial^2 \log
p(\bm{x}^k\,\vert\,\bm{z},\bm\theta)}{\partial\bm{z}^2}
\bigg\vert_{\bm{z}=\bm\mu^k}
$$
is the \emph{observed Fisher information metric}
(FIM)\footnote{The FIM is mostly computed for parameters on a statistical manifold
to describe parameter sensitivity. In contrast, we compute the FIM with respect to
the hidden variable $\bm{z}$.}~\cite{amari16} wrt $\bm{z}$
depending on $\bm\theta$.
The approximation is accurate when $q(\bm{z}\,\vert\,\bm{x}^k,\bm\varphi)$
is Gaussian with vanishing centered-moments of order 3 or above.

Assuming the inference network is flexible enough, minimizing $term_2$ alone gives
$\bm\mu^k=(\bm{z}^k)^\star$, $\bm{V}^k=\bm0$,
where $(\bm{z}^k)^\star=\argmax_{\bm{z}} \log p(\bm{x}^k\,\vert\,\bm{z},\,\bm\theta)$
is the maximum likelihood estimation wrt $\bm{x}^k$.
This $(\bm{z}^k)^\star$ is the latent $\bm{z}$ learned by a plain autoencoder.
Hence $term_2$ measures a dissimilarity between $q(\bm{z}\,\vert\,\bm{x}^k,\bm\varphi)$
and the Dirac delta distribution $\delta((\bm{z}^k)^\star)$.
By \cref{thm:lb}, we get the following approximation of the variational bound.
\begin{corollary}\label{thm:col}
Assume the inference network is flexible enough.
Consider a variation of Gaussian VAE, where both $p(\bm{z})$ and
$q(\bm{z}\,\vert\,\bm{x}^k)$ are free Gaussian distributions.
The optimal $\mathcal{L}^\star$ is given by
\begin{small}
\begin{align*}
\mathcal{L}^\star = &\min_{\{\bm\mu^k,\bm{V}^k,\bm\theta\}}
\frac{1}{n} \sum_{k=1}^n
\bigg[-\log p(\bm{x}^k\,\vert\,\bm\mu^k, \bm\theta)
+\frac{1}{2}tr\left(\fim_{\bm\theta}(\bm\mu^k) \bm{V}^k\right)\nonumber\\
&-\frac{1}{2}\log\vert\bm{V}^k\vert \bigg] +
\frac{1}{2}\log\left\vert
\overline{\bm{V}^k} + \overline{\bm\mu^k  (\bm\mu^k)^\intercal} -
\overline{\bm\mu^k}\;(\overline{\bm\mu^k})^\intercal \right\vert,
\end{align*}
\end{small}
where ``$\overline{\,\cdot\,}$'' means averaging over $k=1,\cdots,n$.
\end{corollary}
\begin{remark}
Consider roughly $\bm{V}^k\approx\bm{V}$ and $\bm{V}^z=\overline{\bm\mu^k (\bm\mu^k)^\intercal} -
\overline{\bm\mu^k}\;(\overline{\bm\mu^k})^\intercal$.
The term $\frac{1}{2}tr\left(\fim_{\bm\theta}(\bm\mu^k)\bm{V}\right)$
helps to shrink $\bm{V}$ towards $\bm{0}$ and lets
$\bm{V}$ respect the data manifold encoded in the spectrum of $\fim_{\bm\theta}(\bm\mu^k)$.
The term 
$-\frac{1}{2}\log\vert\bm{V}\vert
+ \frac{1}{2}\log\left\vert \bm{V} + \bm{V}_z\right\vert$
enlarges $\bm{V}$ and lets $\bm{V}$ respect the latent manifold of $\{\bm\mu^k\}$.
This reveals a fundamental trade-off between fitting the input data and generalizing.
\end{remark}
\begin{remark}
In Gaussian VAE with $L=1$, the term $\frac{1}{2}tr\left(\fim_{\bm\theta}(\bm\mu^k) \bm{V}^k\right)$
is inaccurate. Approximating this term can potentially give more effective implementations of VAE.
\end{remark}
\begin{remark}
Using BPEF for $q(\bm{z}\,\vert\,\bm{x}^k,\bm\varphi)$
has the advantage that the information preserved in higher order differentiations
$\frac{\partial^d\log p(\bm{x}\,\vert\,\bm{z},\,\bm\theta)}{\partial\bm{z}^d}$ ($d\ge3$)
is captured.
\end{remark}

\begin{figure}[t]
\centering
\includegraphics[width=.6\textwidth]{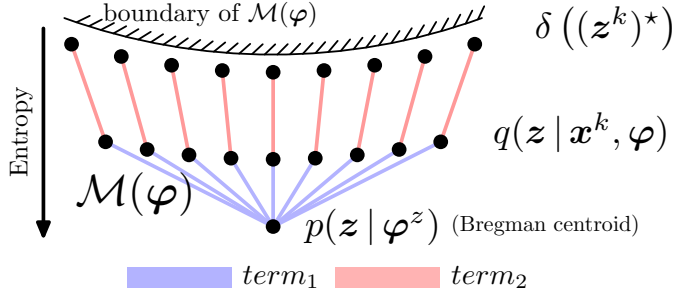}
\caption{A geometric picture of VAE learning. $\mathcal{M}(\bm\varphi)$ is an exponential
family with natural parameters $\bm\varphi$ to which $p(\bm{z})$ and $q(\bm{z}\,\vert\,\bm{x}^k,\bm\varphi)$
both belong. Learning is to minimize the total ``length'' of those colored strings.}\label{fig:ig}
\end{figure}

In summary, \cref{fig:ig} presents the information geometric background of VAE
on the statistical manifold $\mathcal{M}(\bm\varphi)$ which includes $p(\bm{z})$,
$q(\bm{z}\,\vert\,\bm{x}^k)$ and $\delta((\bm{z}^k)^\star)$.
Note that $\delta((\bm{z}^k)^\star)$ is along the boundary of
$\mathcal{M}(\bm\varphi)$, where the variance is 0. For example, a Gaussian
distribution $G(\mu,\sigma)$ with $\sigma\to0$ becomes $\delta(\mu)$.
$\{\delta((\bm{z}^k)^\star)\}$ vary according to maximum
likelihood learning along another statistical manifold $\mathcal{M}(\bm\theta)$.
The cost function $\mathcal{L}$
is interpreted as the geometric compactness of those three sets of distributions.
This is essentially related to the theory of minimum description length~\cite{hinton94,sun15}.

$\fim_{\bm\theta}(\bm{z})$ also gives a lower bound (Cram\'er-Rao bound, \citealt{cramer46})
on the variance of $\bm{z}$, which is given by $\fim_{\bm\theta}^{-1}(\bm{z})$.
In other words, \emph{there is a minimum precision} (or maximum accuracy) that one can achieve
in the inference of $\bm{z}$ given $\bm{x}^k$. Although it is hard to compute exactly, it is 
important to realize the existence of this bound. We give the
following rough estimation. Because each $\bm{z}$ has only single
observation, the FIM of $\bm{z}$ does not scale with $n$. On the other hand,
$\bm\theta$ has $n$ repeated observations. Therefore its FIM 
$\fim(\bm\theta)=\sum_{k=1}^n E_p\left(-\partial^2\log{p}(\bm{x}^k\,\vert\,\bm{z},\bm\theta)/\partial\bm\theta^2\right)$
scales linearly with $n$, meaning that the
precision of $\bm\theta$ scales with $1/\sqrt{n}$. Therefore the precision
of $\bm{z}$ is roughly $\sqrt{n}$ times the precision of $\bm\theta$. Hence
the estimation of $\bm{z}$ should indeed be inaccurate as compared to $\bm\theta$.
This means that the proposed coarse grain technique is not only for
computational convenience, but also has a theoretical background.


\section{Concluding Remarks}\label{sec:con}

Within the variational auto-encoding framework~\cite{vae}, this paper proposed a new method CG-BPEF-VAE.
Among numerous variations of VAE~\cite{burda16,jang17}, CG-BPEF-VAE is featured by
using a universal BPEF density generator in the inference model, and
providing a principled way to simulate continuous densities using discrete
latent variables. For example, to apply another sophisticated distribution
on the latent variable $\bm{z}$, one can employ our CG technique so as to use the reparameterization trick.
This study touches a fundamental problem in unsupervised learning:
\emph{how to build a discrete latent structure to factor information in a continuous representation}?
We provide preliminary results on unsupervised density estimation,
showing performance improvements over the original VAE and category
VAE~\cite{jang17}. An empirical study can be extended to semi-supervised learning. This is ongoing work.

We try to picture an information geometric background of VAE.
Essentially VAE learns on two manifolds $\mathcal{M}(\bm\theta)$
and $\mathcal{M}(\bm\varphi)$, where the cost function can be
geometrically interpreted as a sum of divergences within
a $n$-body system. This potentially leads to new implementations
based on information geometry, e.g., using alternative divergences.


BPEF uses a linear combination of
basis distributions in the $\theta$-coordinates (natural parameters).
Another basic way to define probability distributions is 
mixture modeling, or linear combination in the $\eta$-coordinates (moment parameters).
The coarse grained technique can be extended to a mixture
of BPEF densities, which could more effectively model multi-modal distributions.

\section*{Acknowledgments}

This research is funded by King Abdullah University of Science and Technology.
The experiments are conducted on the Manda cluster provided by Computational
Bioscience Research Center at KAUST.

\bibliographystyle{icml2017}
\bibliography{vae}

\appendix
\section{Proof of Theorem 1}

\begin{proof}
We first prove (1). The mapping $z = \bm{y}^\intercal \bm\zeta$
is from the simplex
$$
\Delta^{R-1} = \left\{
\bm{y}\in\Re^{R}\,:\,\sum_{i=1}^R y_i=1;\;\forall{i},y_i\ge0
\right\}
$$
to the line segment [-1,1]. We therefore define the following subset
$$
\mathcal{S}_{z}
=\left\{ \bm{y}\in\Delta^{R-1} \,:\, \bm{y}^\intercal\bm\zeta =z \right\}
\subset\Delta^{R-1},
$$
where $-1\le{z}\le1$. There we have the following partition scheme
$$
\Delta^{R-1} = \biguplus_{z=-1}^{1} \mathcal{S}_z.
$$
Note $\KL(q(\bm{y}\,\vert\,\bm{x},\bm\varphi):p(\bm{y}\,\vert\,\bm{a}))$
is the KL divergence between two distributions on $\Delta^{R-1}$. By
information monotonicity, 
$$
\KL(q(\bm{y}\,\vert\,\bm{x},\bm\varphi)\,:\,p(\bm{y}\,\vert\,\bm{a}))
\ge
\KL( q(\mathcal{S}_z\,\vert\,\bm{x},\bm\varphi)\,:\,p(\mathcal{S}_z\,\vert\,\bm{a}) ).
$$
where $q(\mathcal{S}_z\,\vert\,\bm{x},\bm\varphi)=
\int_{\bm{y}\in\mathcal{S}_z} q(\bm{y}\,\vert\,\bm{x},\bm\varphi) \dy$
and $p(\mathcal{S}_z\,\vert\,\bm{a})=
\int_{\bm{y}\in\mathcal{S}_z} p(\bm{y}\,\vert\,\bm{a}) \dy$
are coarse grained distributions defined on $[-1,1]$, that is,
$q(\bm{z}\,\vert\,\bm{x},\bm\varphi)$ and $p(\bm{z}\,\vert\,\bm{a})$.

To prove (2), we partition $\bm\Delta^{R-1}$ based on a Voronoi diagram.
Let 
\begin{equation*}
V_r = \{ \bm{y}\in\Delta^{R-1} \,:\, \Vert\bm{y}-\bm{e}_r\Vert_2
\le\Vert\bm{y}-\bm{e}_o\Vert, \forall{o}\neq{r} \}.
\end{equation*}
where $\bm{e}_r\in\Delta^{R-1}$ has the $r$'th bit set to 1 and the rest bits
set to 0.
By the basic property of Concrete distribution (Eq.(8) in the paper),
\begin{align*}
\int_{\bm{y}\in{V}_r}q(\bm{y}\,\vert\,\bm{x},\bm\varphi)\dy
&= P(y_r\ge{y}_o, \forall{o}\neq{r}) =
\frac{\exp(\beta_{jr})}{\sum_{r=1}^R \exp(\beta_{jr})},\\
\int_{\bm{y}\in{V}_r}p(\bm{y}\,\vert\,\bm{a})\dy
&=
\frac{\exp(\alpha_{jr})}{\sum_{r=1}^R \exp(\alpha_{jr})}.
\end{align*}
Then (2) follows immediately from information monotonicity.
\end{proof}

\section{Proof of Theorem 2}

\begin{lemma}
Let $p(\bm{z}) = \exp(\bm\varphi^\intercal \bm{t}(\bm{z}) - F(\bm\varphi))$
be a distribution in an exponential family , then we have 
$\calI(\bm\eta) - \bm\eta^\intercal \bm\varphi + F(\bm\varphi)=0$.
\end{lemma}
\begin{proof}
By definition,
\begin{align*}
\calI(\bm\eta) &= \int p(\bm{z}) \log p(\bm{z}) \dz 
= \int p(\bm{z}) \left(\bm\varphi^\intercal \bm{t}(\bm{z}) - F(\bm\varphi)\right)
\dz\\
&= \bm\varphi^\intercal \int p(\bm{z}) \bm{t}(\bm{z})\dz - F(\bm\varphi)
= \bm\varphi^\intercal \bm\eta - F(\bm\varphi).
\end{align*}
\end{proof}

\begin{proof}
If both $q(\bm{z}\,\vert\,\bm{x}^k)$ and $p(\bm{z})$ are in the same exponential family,
we have
\begin{align*}
p(\bm{z}) &= \exp\left(
\bm{t}^\intercal(\bm{z})\bm\varphi^z - F(\bm\varphi^z)\right),\\
q(\bm{z}\,\vert\,\bm{x}^k)
&= \exp\left(
\bm{t}^\intercal(\bm{z})\bm\varphi^k - F(\bm\varphi^k)\right).
\end{align*}
Therefore
\begin{align}\label{eq:term1}
term_1 =&
\frac{1}{n} \sum_{k=1}^n \KL\left(
q(\bm{z}\,\vert\,\bm{x}^k,\bm\varphi)\,:\,p(\bm{z}\,\vert\,\bm\varphi^z) \right)\nonumber\\
=&
\frac{1}{n} \sum_{k=1}^n
\int q(\bm{z}\,\vert\,\bm{x}^k,\bm\varphi)
\log\frac{q(\bm{z}\,\vert\,\bm{x}^k,\bm\varphi)}
{p(\bm{z}\,\vert\,\bm\varphi^z)} \dz \nonumber\\
=&
\frac{1}{n}\sum_{k=1}^n \left[ \calI(\bm\eta^k) - 
\int q(\bm{z}\,\vert\,\bm{x}^k,\bm\varphi)
(\bm{t}^\intercal(\bm{z})\bm\varphi^z - F(\bm\varphi^z)) \dz
\right] \nonumber\\
=&
\frac{1}{n}\sum_{k=1}^n\left[ \calI(\bm\eta^k)
- (\bm\eta^k)^\intercal(\bm\varphi^z)  + F(\bm\varphi^z) \right].
\end{align}
Because $F(\bm\varphi^z)$ is convex with respect to $\bm\varphi^z$, 
setting its derivative 
\begin{align*}
\frac{\partial{term_1}}{\partial\bm\varphi^z}
= \frac{1}{n}\sum_{k=1}^n
\left[-\bm\eta^k + \frac{\partial F}{\partial\bm\varphi^z}\right]
= \frac{1}{n}\sum_{k=1}^n
\left[-\bm\eta^k + \bm\eta^z\right]\quad\text{(Legendre transformation)}
\end{align*}
to zero gives the \emph{unique} minimizer of $term_1$:
$$
(\bm\eta^z)^\star = \frac{1}{n}\sum_{k=1}^n\bm\eta^k.
$$
Plugging this into Eq.~\ref{eq:term1}, we get
\begin{align*}
term_1^\star &=
\frac{1}{n}\sum_{k=1}^n\left[ \calI(\bm\eta^k) - 
(\bm\eta^k)^\intercal (\bm\varphi^z)^\star
+ F((\bm\varphi^z)^\star) \right]\nonumber\\
&=
\frac{1}{n}\sum_{k=1}^n\left[ \calI(\bm\eta^k)
- ((\bm\varphi^z)^\star)^\intercal \bm\eta^k
+ ((\bm\varphi^z)^\star)^\intercal (\bm\eta^z)^\star
- \calI((\bm\eta^z)^\star)\right]\quad\text{(by the Lemma)}\nonumber\\
&=
\frac{1}{n}\sum_{k=1}^n\calI(\bm\eta^k)
- ((\bm\varphi^z)^\star)^\intercal \frac{1}{n}\sum_{k=1}^n \bm\eta^k
+ ((\bm\varphi^z)^\star)^\intercal (\bm\eta^z)^\star
- \calI((\bm\eta^z)^\star)\nonumber\\
&=
\frac{1}{n}\sum_{k=1}^n\calI(\bm\eta^k)
- \calI\left( \frac{1}{n}\sum_{k=1}^n \bm\eta^k \right).
\end{align*}
By the above analysis, $term_1 \ge term_1^\star$, and the ``='' holds if and only if 
$\bm\eta^z = (\bm\eta^z)^\star$. The second ``$\ge$'' is straightforward from the fact
that $\calI$ is a convex function in the coordinate system $\bm\eta$.

If $p(\bm{z})$ is non-parametric, then
\begin{align}\label{eq:term12}
term_1 =&
\frac{1}{n} \sum_{k=1}^n
\int q(\bm{z}\,\vert\,\bm{x}^k,\bm\varphi)
\log\frac{q(\bm{z}\,\vert\,\bm{x}^k,\bm\varphi)}
{p(\bm{z}\,\vert\,\bm\varphi^z)} \dz \nonumber\\
=&
\frac{1}{n}\sum_{k=1}^n \left[ \calI(\bm\eta^k) - 
\int q(\bm{z}\,\vert\,\bm{x}^k,\bm\varphi)
\log p(\bm{z}\,\vert\,\bm\varphi^z) \dz \right]
\nonumber\\
=&
\frac{1}{n}\sum_{k=1}^n \calI(\bm\eta^k)
- \int \frac{1}{n}\sum_{k=1}^n
q(\bm{z}\,\vert\,\bm{x}^k,\bm\varphi) \log p(\bm{z}\,\vert\,\bm\varphi^z) \dz.
\end{align}
Therefore, $term_1$ is minimized at 
$p(\bm{z}\,\vert\,(\bm\varphi^z)^\star)
= \frac{1}{n}\sum_{k=1}^n q(\bm{z}\,\vert\,\bm{x}^k,\bm\varphi)$.
Plugging this minimizer into the above Eq.~\ref{eq:term12}, we get
\begin{align*}
term_1^\star = 
\frac{1}{n}\sum_{k=1}^n \calI(\bm\eta^k)
- \int \frac{1}{n}\sum_{k=1}^n q(\bm{z}\,\vert\,\bm{x}^k,\bm\varphi) \log
\left[\frac{1}{n}\sum_{k=1}^n q(\bm{z}\,\vert\,\bm{x}^k,\bm\varphi)\right] \dz.
\end{align*}
Note that the mixture model
$\frac{1}{n}\sum_{k=1}^n q(\bm{z}\,\vert\,\bm{x}^k,\bm\varphi) $
is outside the exponential family $\mathcal{M}(\bm\varphi)$.
\end{proof}

\section{The Effect of the Dimensionality Reduction Layer of CG-BPEF-VAE}

Fig.~\ref{fig:bound1} shows the $\KL(p(\bm{z}):\mathtt{Uniform})$ (KL(z))
and $\KL(\bm\alpha:\mathtt{Uniform})$ (KL(category)) when $\bm\alpha$
is generated by Dirichlet distributions with different configurations.
In all cases, $\KL(\bm\alpha:\mathtt{Uniform})$  is lower bounded by 
$\KL(p(\bm{z}):\mathtt{Uniform})$ for small temperature.
\section{Visualization of the Concrete Distribution}

Fig.~\ref{fig:con1}, Fig.~\ref{fig:con2} and Fig.~\ref{fig:con3} 
show Concrete densities generated by random sampling. For each experiment (sub-figure),
we generate $10^6$ Concrete samples and plot the resulting density.
There are very high density regions near the corner (the red region),
which are cropped so that the visualization is clear.

An interesting observation is that the density will ``leak'' to the simplex faces
if $T$ is small, although in this case the density will concentrate on the corners.
Therefore it may not always be good to choose a small $T$. This is ongoing study.

\section{Details of the Convolutional Layers}

We used convolutional layers on the SVHN dataset.
The encoder is specified by
\begin{itemize}
\item Input: $1\times32\times32$ (RGB is averaged into 1 channel)
\item Convolutional layer: 32 ($5\times5$) filters, with ReLU activation and no padding\\
($\to32\times28\times28$)
\item Pooling layer: $2\times2$ filter with a stride of 2 and no padding zeros\\
($\to32\times14\times14$)
\item Convolutional layer: 64 ($5\times5$) filters, with ReLU activation and no padding\\
($\to64\times10\times10$)
\item Pooling layer: $2\times2$ filter with a stride of 2 and no padding\\
($\to64\times5\times5$)
\item Convolutional layer: 128 ($5\times5$) filters, with RELU activation and no padding\\
($\to128\times1\times1$)
\end{itemize}
The decoder is specified by
\begin{itemize}
\item A dense linear layer with RELU activation to transform the dimension to 128
\item Transposed convolutional layer: 64 ($5\times5$) filters with stride 4; with RELU activation\\
($\to64\times4\times4$)
\item Transposed convolutional layer: 32 ($5\times5$) filters with stride 2; with RELU activation\\
($\to32\times8\times8$)
\item Transposed convolutional layer: 16 ($5\times5$) filters with stride 2; with RELU activation\\
($\to16\times16\times16$)
\item Transposed convolutional layer:  1 ($5\times5$) filters with stride 2; without non-linear activation\\
($\to1\times32\times32$)
\end{itemize}

\begin{figure}[p]
\centering
\includegraphics[width=\textwidth]{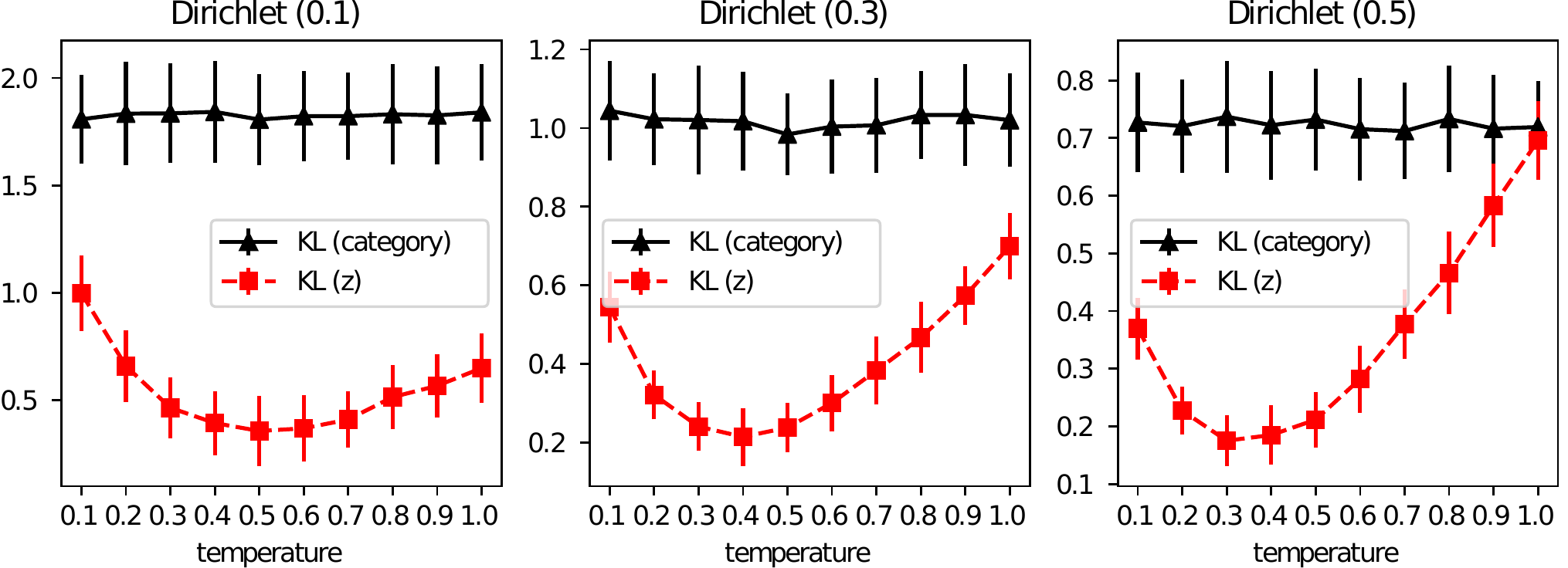}
\includegraphics[width=\textwidth]{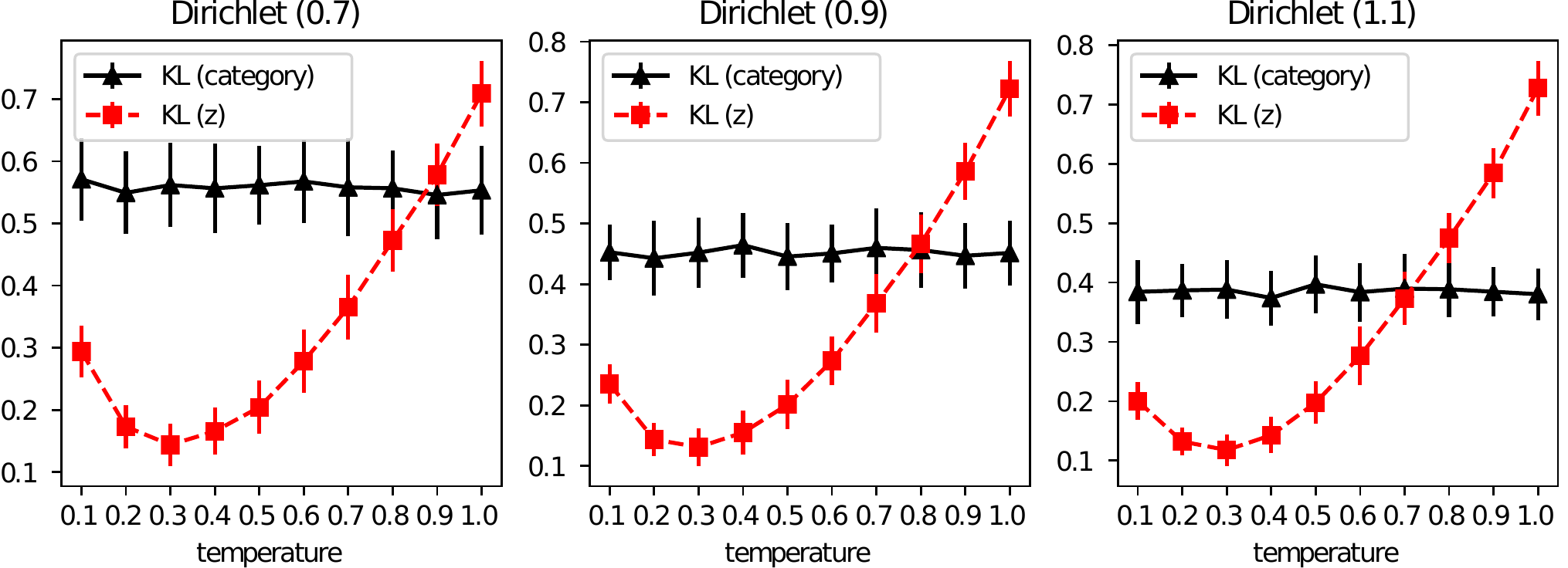}
\includegraphics[width=\textwidth]{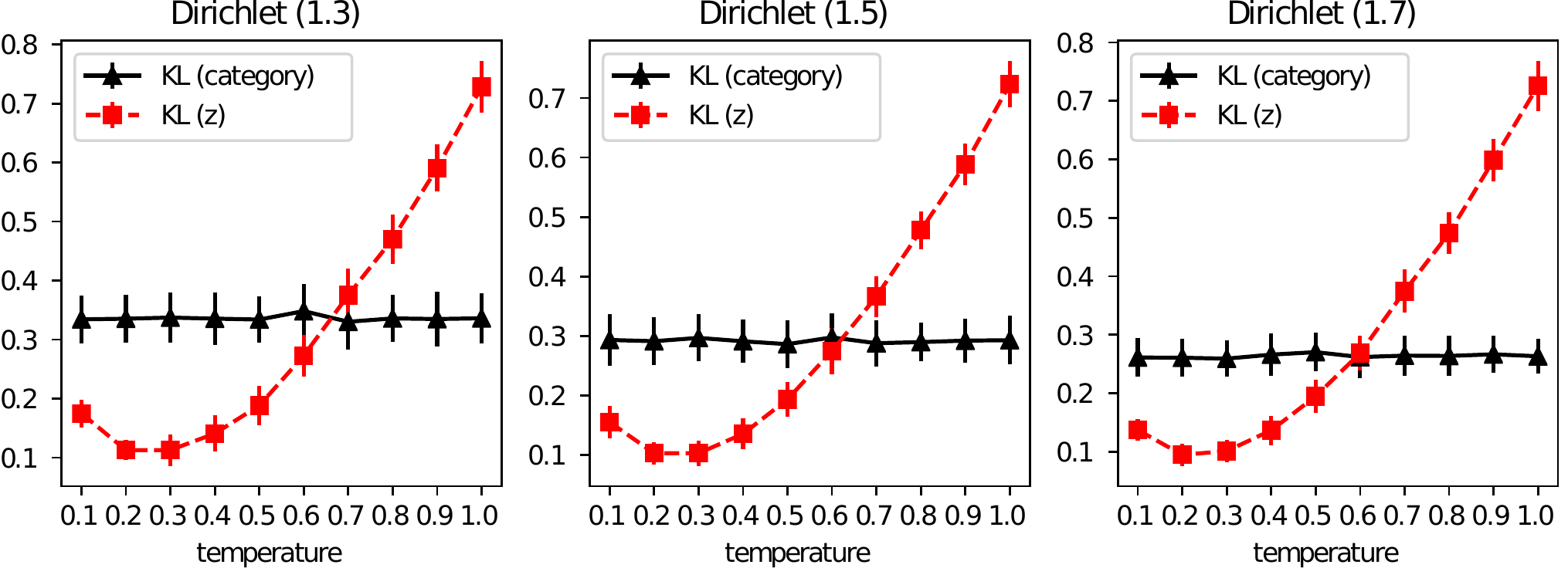}
\caption{KL divergence between category distributions v.s. KL divergence
between corresponding Gumbel distributions ($\bm{y}$) after dimensionality reduction ($\bm{z}$).
The figures show mean$\pm$standard deviation of KL. Both of KL divergences are computed by
discretizing $(-1,1)$ into $R=100$ intervals.}\label{fig:bound1}
\end{figure}

\begin{figure}[p]
\centering
\includegraphics[width=\textwidth]{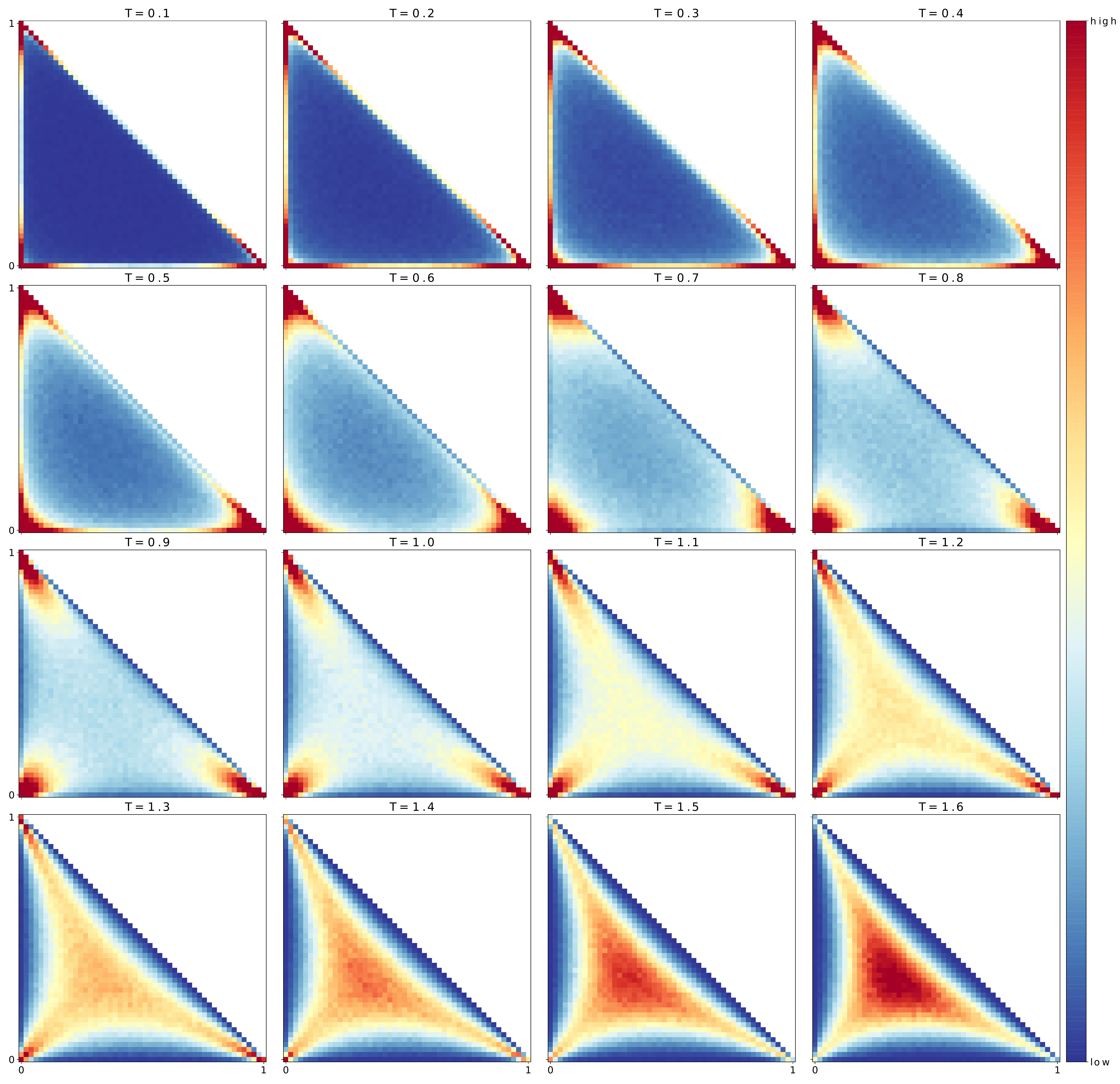}
\caption{The density of $\mathrm{Con}(\bm\alpha,T)$, where $\bm\alpha=(1/3,1/3,1/3)$ and $T=0.1,0.2,\cdots,1.6$
(from left to right, from top to bottom). Best viewed in color.}
\label{fig:con1}
\end{figure}

\begin{figure}[p]
\centering
\includegraphics[width=\textwidth]{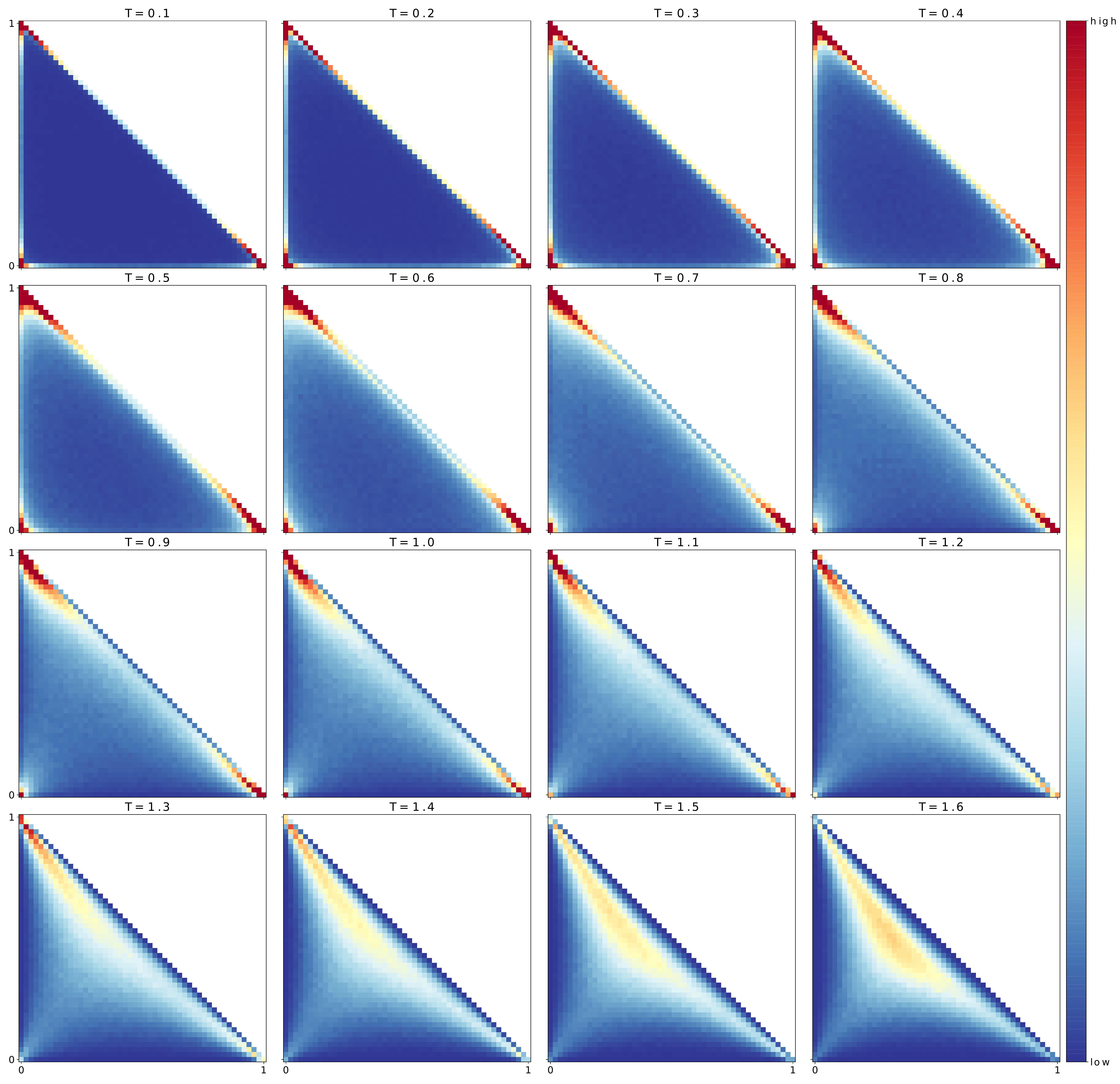}
\caption{The density of $\mathrm{Con}(\bm\alpha,T)$, where $\bm\alpha=(1/2,1/3,1/6)$ and $T=0.1,0.2,\cdots,1.6$
(from left to right, from top to bottom). Best viewed in color.}
\label{fig:con2}
\end{figure}

\begin{figure}[p]
\centering
\includegraphics[width=\textwidth]{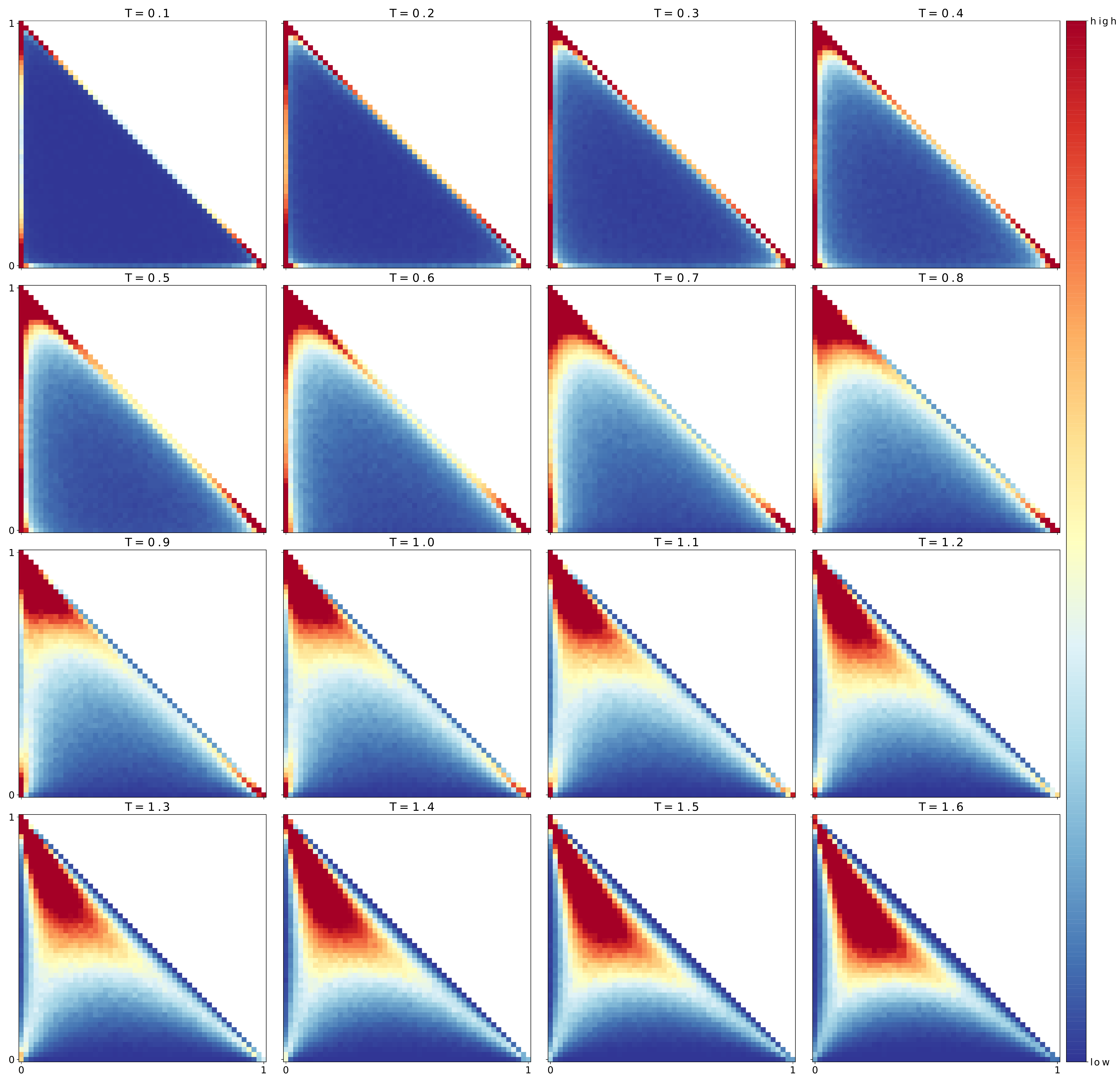}
\caption{The density of $\mathrm{Con}(\bm\alpha,T)$, where $\bm\alpha=(2/3,1/6,1/6)$ and $T=0.1,0.2,\cdots,1.6$
(from left to right, from top to bottom). Best viewed in color.}\label{fig:con3}
\end{figure}

\end{document}